\documentclass[smallcondensed]{svjour3}  
\smartqed 
\usepackage
[
        a4paper,
        left=3cm,
        right=3cm,
        top=3cm,
        bottom=3cm,
]
{geometry}

\usepackage[american]{babel}
\usepackage[square,sort,comma,numbers]{natbib}
\usepackage{amsmath,amsthm,mathtools}
\usepackage{microtype}
\usepackage{graphicx}
\usepackage{subfigure}
\usepackage{booktabs} 
\usepackage{float,xspace}
\usepackage{balance}

\usepackage{pbox}
\usepackage{cellspace}
\usepackage{placeins}
\usepackage{dblfloatfix}

\usepackage{hyperref,xspace}

\usepackage{color}
\usepackage{framed}

\usepackage{tabu}

\usepackage[boxed,linesnumbered,ruled]{algorithm2e}
\usepackage{xcolor}
\usepackage{cite}
\usepackage{flushend}
\newcommand\blfootnote[1]{%
  \begingroup
  \renewcommand\thefootnote{}\footnote{#1}%
  \addtocounter{footnote}{-1}%
  \endgroup
}

\usepackage{hyperref}
\usepackage{amsfonts}
\usepackage{amssymb}
\usepackage{color}
\usepackage{framed}

\usepackage{bbm}
\usepackage{verbatim}
\usepackage{booktabs,siunitx}

\newcommand\numberthis{\addtocounter{equation}{1}\tag{\theequation}}

\newcommand{\minhash}{\mathrm{minHash}}
\newcommand{\RMSE}{\mathrm{RMSE}}

\newcommand{\lifthash}{\mathrm{liftHash}}
\newcommand{\drophash}{\mathrm{dropHash}}
\newcommand{\mdrophash}{\mathrm{multipleDropHash}}
\newcommand{\mlifthash}{\mathrm{multipleLiftHash}}

\newtheorem{example}[theorem]{Example}

\usepackage{thmtools, thm-restate}
\declaretheorem[name=Proposition]{prop}

\usepackage[utf8]{inputenc} 
\usepackage[T1]{fontenc}    
\usepackage{hyperref}       
\usepackage{url}            
\usepackage{booktabs}       
\usepackage{amsfonts}       
\usepackage{nicefrac}       
\usepackage{microtype}      
\usepackage[inline]{enumitem}
\setlist[enumerate,1]{label=\textit{\alph*)}}
\usepackage{multirow}

\SetCommentSty{mycommfont}

\usepackage{tikz} 

\def\makeheadbox{\relax}

\title{Minwise-Independent Permutations with Insertion and Deletion of Features}

\author{Rameshwar Pratap~\footnote{Corresponding author.}\and
        Raghav Kulkarni}

\authorrunning{Pratap \and Kulkarni} 

\institute{
          Rameshwar Pratap \at
          IIT Hyderabad, Telangana, India \\
          \email{rameshwar@cse.iith.ac.in}      
          \and
         Raghav Kulkarni\at
            Chennai Mathematical Institute, Chennai, TamilNadu, India \\
            \email{kulraghav@gmail.com}
}
\begin{document}

\maketitle

\begin{abstract}

In their seminal work, Broder \textit{et. al.}~\citep{BroderCFM98} introduces the $\mathrm{minHash}$ algorithm that computes a low-dimensional sketch of high-dimensional binary data that closely approximates pairwise Jaccard similarity. Since its invention, $\mathrm{minHash}$ has been commonly used by practitioners in various big data applications.  Further, the data is dynamic in many real-life scenarios, and their feature sets evolve over time.  We consider the case when features are dynamically inserted and deleted in the dataset.   We note that a naive solution to this problem is to repeatedly recompute $\mathrm{minHash}$ with respect to the updated dimension. However, this is an expensive task as it requires generating fresh random permutations. To the best of our knowledge, no systematic study of $\mathrm{minHash}$ is recorded in the context of dynamic insertion and deletion of features. In this work, we initiate this study and suggest algorithms that make the $\mathrm{minHash}$ sketches adaptable to dynamic insertion and deletion of features. We show a rigorous theoretical analysis of our algorithms and complement it with extensive experiments on several real-world datasets. Empirically we observe a significant speed-up in the running time while simultaneously offering comparable performance with respect to running $\mathrm{minHash}$ from scratch. Our proposal is efficient, accurate, and easy to implement in practice.

\end{abstract}


\section{Introduction}

 {\color{black} The seminal work of Broder~\textit{et.~al.}~\citep{BroderCFM98}  suggests the $\minhash$ algorithm that computes a low-dimensional representation (or \textit{sketch}) of the high-dimensional binary data that closely approximates the underlying pairwise Jaccard similarity.} The Jaccard similarity between two binary data points $X, Y \in\{0,1\}^d$ is defined as $|S(X)\cap S(Y)|/|S(X)\cup S(Y)|$ where $S(X) = \{i : x_i = 1\}$ and $x_i$ is the $i$-th feature of $X$.  For a $d$-dimensional permutation $\pi$ chosen uniformly at random, and a data point $X \in\{0,1\}^d$, the $\minhash$~\citep{BroderCFM98} is defined  as follows:  $$\minhash_\pi(X) = \min_{s \in S(X)} \pi(s).$$ \blfootnote{We note that binary vectors and sets give two equivalent representations of the same data object. We elaborate on it as follows: Consider our data elements are a subset of a fixed universe. In the corresponding binary representation, we can generate a vector whose dimension is the size of the universe, where for each possible element of the universe, a feature position is designated. To represent a set into a binary vector, we label each element's location with $1$ if it is present in the set, and $0$ otherwise.} {\color{black} For a pair of points  $X, Y \in\{0,1\}^d$, $\minhash$~\citep{BroderCFM98} offers the following guarantee
 \begin{align*}
    \Pr[\minhash_{\pi}(X)=\minhash_{\pi}(Y)]&=\frac{|S(X)\cap S(Y)|}{|S(X)\cup S(Y)|}.
 \end{align*}}
  {\color{black}In fact, the same guarantee holds for a restricted class of permutations called min-wise independent permutations~\citep{10.5555/314500.314600,BroderCFM98,DBLP:journals/rsa/MatousekS03} . The above characteristic demonstrates the locality-sensitive nature (LSH)~\citep{IM98} of $\minhash$, and as a consequence, it can be effectively used for the approximate nearest neighbour search problem.}  $\minhash$ is successfully applied in several real life applications such as computing document similarity~\citep{broder1997resemblance}, itemset mining~\citep{CohenDFGIMUY01,DBLP:conf/cocoon/BeraP16}, faster de-duplication~\citep{BroderCPM00}, all-pair similarity search~\citep{BayardoMS07}, document clustering~\citep{broder2000method}, building recommendation engine~\citep{recom_engine}, near-duplicate image detection~\citep{DBLP:conf/bmvc/ChumPZ08}, web-crawling~\citep{10.1145/1242572.1242592,10.1145/1148170.1148222}, genomics~\citep{10_24072_pcjournal_37,OndovBrianD.2016Mfga,Brown2016,BerlinKonstantin2015CAlg},  large scale graph hashing~\citep{liu2014discrete,gibson2005discovering} to name a few.

 {\color{black} This work considers the scenario where features are dynamically inserted and/or deleted from the input. We emphasize that this natural setting may arise in many applications. Consider the \textit{``Bag-of-Word" (BoW)} representation of text, where first, a dictionary is created using the important words present in the corpus such that each word present in the dictionary corresponds to a feature in the representation. Consequently, the embedding of each document is generated using this dictionary based on the frequency of the words present.  Consider the downstream application where the task is to compute pairwise Jaccard similarities between these documents, and  the dimensionality of the \textit{BoW} representation is high due to the large dictionary size. We can  use  $\minhash$ to compute the low-dimensional sketch of input documents.  It is quite natural to assume that the dictionary is evolving; new words are inserted into the dictionary, and unused words are deleted.  One evident approach to handle such a dynamic scenario is to run the $\minhash$ from scratch on the updated dictionary, which is expensive since it involves generating  fresh min-wise independent (random) permutations.} Note that during the insertion/deletion of features in the dataset, we consider inserting/deleting the same features in all the data points. To clarify this further, let $\mathcal{D} = \{X_i\}_{i=1}^n$ be our dataset, where $X_i\in \{0, 1\}^d$.  Considering the addition/removal of the $j$-th feature, the $j$-th feature gets inserted/deleted in the point $X_i$. Similarly, the corresponding $j$-th feature is inserted/deleted in all the remaining points in $\mathcal{D}$. Note that we don’t consider the case when data points are dynamically inserted or deleted in the dataset. \\

\noindent\textbf{Problem statement: $\mathbf{\minhash}$ for dynamic insertion and deletion of features:
} In this work, we focus on making $\minhash$ adaptable to dynamic feature insertions and deletions of features. We note that the insertion/deletion of features dynamically leads to the expansion/shrink of the data dimension.

We note that in practice a $d$ dimensional permutation required for $\minhash$ is  generated via the universal hash function $h_d(i)=((ai+b)\mod p)\mod d$, where $p$ is a large prime number, and $a, b$ are randomly sampled from $\{0, 1, \ldots p-1\}$; typically $((ai+b)\mod p)> d$ \footnote{These hash functions are called universal hash functions. Readers may  refer to Chapter $11$ of \citep{DBLP:books/daglib/0023376} for details.}. This hash function generates permutations via mapping each index $i\in [d]$ to another index $[d]$ that can be used to compute the $\minhash$ sketch.
{\color{black} We note that in the case of dynamic insertions/deletion of features, even using universal hash functions to compute the $\minhash$ sketch doesn’t give an efficient solution.} 
We illustrate it as follows.  Suppose we have a  $\minhash$ sketch of data points using the hash function $h_d(.)$. Consider the  case of feature insertion,  where the dimension $d$ increases to $d+1$, and therefore, we require a hash function $h_{d+1}(.)$ to generate a $(d+1)$-dimensional permutation. Note that the permutation generated via $h_{d+1}(.)$ can potentially be different on several values of $i \in [d+1].$ Therefore,  just computing $h_{d+1}(d+1)$, taking the corresponding input feature, and taking the minimum of this quantity with the previous $\minhash$ would not suffice to compute $\minhash$ after feature insertion. This re-computation seems to take $O(d)$ in the worst case if implemented naively. A similar argument also holds in the case of feature deletion.


\subsection{Our Contribution:}
In this work, we consider the problem of making $\minhash$ adaptable to dynamic insertions and deletions of features. We focus on cases where features are inserted/deleted at randomly chosen positions from $1$ to $d$. We argue that this is a natural assumption that commonly occurs in practice. For example, in the context of \textit{BoW}, a word's position in the dictionary is determined via a random hash function that randomly maps it to a position from $1$ to $d$. Therefore, when a new word is added to the dictionary, its final position in the representation appears as a random position (from $1$ to $d$). A similar argument is also applicable for feature deletion. With this motivation and context, we summarize our contributions as follows:

\begin{itemize}
\item \textbf{Contribution 1:} We present algorithms  (Section~\ref{sec:feature_insertion}) that makes $\minhash$ sketch adaptable to single/multiple feature insertions. Our algorithm takes the current permutation and the corresponding $\minhash$ sketch; values and positions of the inserted features as input and outputs the $\minhash$ sketch corresponding to the updated dimension.


\item \textbf{Contribution 2:} We also suggest algorithms (Section~\ref{sec:feature_deletion}) that makes $\minhash$ sketch adaptable  for single/multiple feature deletions. It takes the data points, current sketch, and permutations used to generate the same; positions of the deleted features and outputs the $\minhash$ sketch corresponding to the updated dimension.
\end{itemize}
 
 Our work leaves the possibility of some interesting open questions: to propose algorithms when features are inserted or deleted adversarially (rather than uniformly at random from $1$ to $d$, as considered in this work). We hope that our techniques can be extended to handle this situation.


\subsection{Our techniques and their advantages: }


A major benefit of our results is that they do not require generating fresh random permutations corresponding to the updated dimension (after feature insertions/deletions) to compute the updated sketch. We implicitly generate a new permutation (required to compute the sketch after feature insertion/deletion) using the old $d$-dimensional permutation, and also show that it satisfies the min-wise independence property. We further give simple and efficient update rules that take the value and position of inserted/deleted features, and output the updated $\minhash$ sketch. To show the correctness of our result, we prove that the sketch obtained via our update rule is the same as obtained via computing $\minhash$ from scratch using the implicitly generated permutation as mentioned above.

For both insertions and deletion cases, our algorithms give significant speedups in dimensionality reduction time while offering almost comparable accuracy with respect to running $\minhash$ from scratch. We validate this by running extensive experiments on several real-world datasets (Section~\ref{sec:experiments} and Table \ref{tab:speedup}). We want to emphasize that our algorithms can also be easily implemented when permutations are generated via random hash functions.



\subsection{Applicability  of our result in other sketching algorithms for Jaccard similarity:}
 We note that there are several improved variants  of $\minhash$ are known  such as \text{one-permutation hashing}~\citep{one_permutation_hashing,DBLP:conf/uai/Shrivastava014}, \text{$b$-bit minwise hashing}~\citep{b-bit,DBLP:conf/internetware/0001SK13}, \text{oddsketch}~\citep{oddsketch} that offer space/time efficient sketches.  We would like to highlight that our algorithms can be easily adapt to these improved variants of $\minhash$, in case of dynamic insertion and deletion of features. We briefly discuss it as follows: One permutation hashing divides the permuted columns evenly into $k$ bins. For each data point, the sketch is computed by picking the smallest nonzero feature location in each bin. In the case of dynamic settings, our algorithms can be applied in the bin where features are getting inserted/deleted. Both $b-$bit minwise hashing~\citep{b-bit} and oddsketch~\citep{oddsketch} are two-step sketching algorithms. In their first step,  the $\minhash$  sketch of the data points is computed. In the second step of {$b$-bit minwise hashing}, the last  $b$-bits (in the binary representation) of each $\minhash$  signature is computed, whereas in the second step of oddsketch, one bit of each $\minhash$  sketch is computed using their proposed hashing algorithm. As both of these results compute the $\minhash$ sketch in their first step, we can apply our algorithms to compute the $\minhash$ sketch in case of feature insertion/deletion. This will make their algorithms adaptable to dynamic feature insertions and deletions.

 Recently,  some hashing algorithms have been proposed that closely estimate the pairwise Jaccard similarity~\citep{ChristianiP17,McCauleyMP18,DBLP:conf/icde/ChristianiPS18} without computing their $\minhash$ sketch. However, to the best of our knowledge, their dynamic versions (that can handle dynamic insertions/deletions of features) are unknown. Several improvements of the LSH algorithm~\citep{SundaramTSMIMD13} have been proposed that are adaptable to the dynamic/streaming framework. However, a significant difference is in the underlying problem statement. These results aim to handle dynamic insertion and deletions of data points, whereas we focus on dynamic insertions and deletions of the features.

\paragraph{Organization of the paper:}
In Section~\ref{sec:background} we present the required technical background. In Section~\ref{sec:feature_insertion} and Section~\ref{sec:feature_deletion} we present our algorithms for feature insertion and feature deletion respectively. We summarize the results of our experiments in Section~\ref{sec:experiments}. Finally we note the concluding remarks in Section~\ref{sec:conclusion}.

 \section{Background} \label{sec:background}
 \begin{definition}[Minwise Independent Permutations~\citep{BroderCFM98}] \label{def:minwise}
Let $S_d$ be the set of all permutation on $[d]$. We say that $F \subseteq S_d$ (the symmetric group) is min-wise independent if for any set $U \subseteq [d]$ and any $u \in U$, when $\pi$ is chosen at random in $F$, we have
\begin{align*}
    \Pr[\min\{\pi(U)\} &= \pi(u)]=1/|U|.                  \numberthis \label{eq:minwise}
\end{align*}
For a permutation $\pi \in F$ chosen at random and a set $U \subseteq [d]$, Broder \textit{et.al.}~\citep{BroderCFM98} define $\minhash$ as follows  $\minhash_\pi(U) = \arg\min_u \pi(u)$ for $u \in U$.
 For two data points, $U,V \subseteq [d]$, and $\pi$ is chosen at random in $F$, due to $\minhash$ we have
 \begin{align*}
    \Pr[\minhash_{\pi}(U)=\minhash_{\pi}(V)]&=|U \cap V|/|U \cup V| \numberthis \label{eq:minwise_jaccard}.
 \end{align*}
  \end{definition}

 \section{Algorithm for feature insertion}\label{sec:feature_insertion}
 \begin{table*}
\caption{\footnotesize{Notations}\label{tab:notations}} 
    \centering
    \scalebox{0.83}{
      \noindent\begin{tabular}{lc||lc}
    \hline
    Data dimension & $d$ & Input data point $\{0,1\}^d$ or input set & $X$ \\
    \hline
    Set $\{1,\ldots, d\}$ & $[d]$ &Data point after feature insertion $\{0,1\}^{d+1}$  & $X'$ \\
    \hline
   Position of the inserted feature & $m$     & Original $d$-dim. permutation $(a_1,\ldots,a_d)$ \textit{s.t.} $a_i \in [d]$  & $\pi$\\
    \hline
    Value  of the inserted feature & $b$     & Lifted $(d+1)$-dim. permutation $(a’_1,\ldots a’_{d+1})$ \textit{s.t.} $a’_i \in[d+1]$  & $\pi'_m$\\
    \hline
     No. of $1'$s in $X$ & $|X|$     &Set of non-zero indices of $X$, \textit{i.e.}, 
 $\{i | x_i = 1\}$&$J$\\
    \hline
     Size of the set $J$  & $|J|$     & $\minhash$ of $X$ with $\pi$, \textit{i.e.}, $\minhash_{\pi}(X)$&$h_{old}$\\
     \hline
\end{tabular}
}
\end{table*}
 We first give our algorithm for a single feature insertion. We  discuss it in the following subsection. 
\subsection{One feature insertion at a time -- $\mathrm{liftHash}$:}\label{subsec:liftHash}
The $\mathrm{liftHash}$ (Algorithm~\ref{alg:update_sketch}) is our main algorithm for updating the sketch of data points consisting of binary features. It takes a $d$ dimensional permutation $\pi$ and the corresponding $\minhash$ sketch $h_{old}$ $\pi$  as input. In addition, it takes an index $m$ and a bit value $b$, corresponding to the position and the value of the binary feature to be inserted, respectively, and  outputs updated  hash value $h_{new}$. We show that $h_{new}$ corresponds to a $\minhash$ sketch of the updated feature vector. In order to show this, we use $\mathrm{liftPerm}$ (Algorithm~\ref{alg:lift}) which extends the original permutation $\pi$ to  a $(d+1)$ dimensional min-wise independent permutation. Note that the $\mathrm{liftPerm}$ algorithm is used solely for the proof and not required in the $\mathrm{liftHash}$ algorithm.

The main intuition of our algorithm is that we can (implicitly) generate a new $(d+1)$-dimensional permutation by  reusing the old $d$-dimensional permutation (Algorithm~\ref{alg:lift}), and can update the corresponding $\minhash$ \textit{w.r.t.} the new $(d+1)$-dimensional  permutation via a simple update rule (Algorithm~\ref{alg:update_sketch}).
Consider a $d$ dimensional input vector $X = (x_1, x_2, \ldots, x_d).$ A permutation $\pi$ of $\{1, 2, \ldots, d\}$ can be thought of as imposing the following ordering on the indices of $X$: $\pi(1), \pi(2), \ldots, \pi(d)$. After feature insertion, we want the (implicit) $\mathrm{liftPerm}$ algorithm to generate a new permutation $\pi'$ of $\{1, 2, \ldots, d+1\}$ that still maintains the ordering that was imposed by $\pi$. We show that such an extension is achievable with high probability assuming (i) feature insertion is happening at a random position and (ii) our binary feature vector is sparse.
This helps us guarantee (with high probability) that $\pi'$ is min-wise independent if $\pi$ is min-wise independent (see Theorem~\ref{thm:lift}). Finally, we show that the sketch obtained by the $\mathrm{liftHash}$ algorithm is the same one produced by applying the $\minhash$  with respect to the output $\pi'$ of the $\mathrm{liftPerm}$ algorithm (see Theorem~\ref{thm:update_sketch}).


 \begin{algorithm} 
 \textbf{Input:~}{ $d$-dim permutation $\pi$, a number $r$.}\\
 \textbf{Output:~}{$(d+1)$-dim. permutation ${\pi}'$.}\\
   \For{$i\in \{1,\ldots,d+1\}$}
   {
   \eIf{$i\leq r$} 
   {$\pi’(i) = \pi(i)$}
   {$\pi’(i) = \pi(i-1)$ \label{alg:line7liftPerm}} 
   }
   \For{ $i \in \{1,\ldots, d+1\} / \{r\}$}
   {
  \If{${\pi}’(i) \geq  {\pi’}(r)$}
   {  \label{alg:line11liftPerm}
   ${\pi}’(i) = {\pi}’(i) + 1$ \label{alg:line12liftPerm}
   }
   }
   \Return ${\pi}’$\\
\caption{$\mathrm{liftPerm(\pi, r)}$.}\label{alg:lift}
  \end{algorithm}
  \begin{algorithm} 
  \SetNoFillComment
    \textbf{Input:~}{ $h_{old}:=\mathrm{minHash}_{\pi}(X)$,   $\pi$,  $m \in [d]$,   $b \in \{0, 1\}$.}\\
 \textbf{Output:~}{$h_{new}:=\mathrm{liftHash}(\pi, m,b,h_{old})$.}\\%
{ Denote $ a_m=\pi(m)$.} \\
\tcc{$m$ is the position of the inserted feature}
 \eIf{$h_{old} < a_m$}
 {$ h_{new} = h_{old}$}{
 \If{$b=1$}{$h_{new}  = a_m$ \label{alg:line8liftDrop}}
 \If{$b=0$}{$h_{new}  = h_{old} + 1$ \label{alg:line11liftDrop}}
}
\Return $h_{new}$
\caption{$\mathrm{liftHash}(\pi, m,b,h_{old})$.}\label{alg:update_sketch}
\end{algorithm}

We illustrate our algorithm with the following example and then state its proof of correctness.

 \begin{example}
  We illustrate our Algorithms using the following example. We assume that the index count starts with $1$. 
Let \(X =[1,0,0,1,0,1,0]\) be the data point, and \(\pi = 
[6,3,1,7,2,5,4]\) be the original permutation. Then  $\minhash_{\pi}(X)$ is $5$.  Further, let us assume that we insert the value \(b=1\) at the index \(m=2\). Therefore $a_m=\pi(m)=3$. The updated value  \(X'=[1,1,0,0,1,0,1,0]\) and due to Algorithm~\ref{alg:lift} by setting $r=m=2$, we obtain  \(\pi'_m=[7,3,4,1,8,2,6,5]\). We calculate the value of $h_{new}$ outputted by Algorithm~\ref{alg:update_sketch}: as  
  \( h_{old}=5 > a_m = 3 \) and $b=1$, then  we have $h_{new}=\mathrm{liftHash}(\pi, m,b,h_{old})  = a_m=3$. Further, $\minhash_{\pi'_m}(X')=3$. Therefore, we have $h_{new}=\minhash_{\pi'_m}(X')$.  
\end{example}
 
The following theorem gives proof of correctness of Algorithm~\ref{alg:lift}, and shows that the permutation $\pi’$ outputted by the algorithms satisfies the minwise independent property (Definition~\ref{def:minwise}), with high probability. At a high-level proof of Theorem~\ref{thm:lift} relies on showing the bijection between the ordering on the indices of  $X$ by the original $d$-dimensional permutation $\pi$,  and $(d+1)$-dimensional permutation $\pi’$. We show that this bijection holds with  probability $1$ when inserted feature value $b=0$, and holds with a high probability when $b=1$.

 \begin{restatable}{theorem}{insertliftPerm}\label{thm:lift}
 Let  $\pi= (a_1 \ldots, a_d)$ be a minwise independent permutation, where $a_i \in [d]$, and $r$ be a random number from $[d]$. Let $\pi$ and $r$ be the input to Algorithm~\ref{alg:lift}. 
 Then  for any $X\in \{0,1\}^d$ with $|X|\leq k$, the permutation $\pi’ = (a’_1 \ldots, a’_{d+1})$, where $a’_i \in [d+1],$ obtained from Algorithm~\ref{alg:lift} satisfies the condition stated in Equation~\eqref{eq:minwise}   of Definition~\ref{def:minwise}, with  probability at least $1-O(k/d)$. 
 \end{restatable}

 Theorem~\ref{thm:update_sketch} gives a proof of correctness of Algorithm~\ref{alg:update_sketch}. We show that the sketch outputted by Algorithm~\ref{alg:update_sketch} is the same as obtained by running $\minhash$ using the $(d+1)$-dimensional permutation obtained by Algorithm~\ref{alg:lift} on the updated data point after one feature insertion.

 \begin{restatable}{theorem}{insertliftHash}\label{thm:update_sketch}
 Let $\pi’_{m}$ be the $(d+1)$-dimensional permutation outputted  by Algorithm~\ref{alg:lift} by setting $r=m$. Then, the  sketch obtained from Algorithm~\ref{alg:update_sketch} is exactly the same to the  sketch  obtained with  the permutation ${\pi}’_m$ on $X'$, that is, $h_{new}:=\mathrm{liftHash}(\pi, m,b,h_{old}) = \mathrm{minHash}_{\pi'_m}(X')$ .
 \end{restatable}
 \begin{remark}
 We remark that in order to compute the $\minhash$ sketch of  $X’$, Algorithm~\ref{alg:update_sketch} requires only $h_{old}$, $b, m$, the value of $\pi(m)$.  Whereas \textit{vanilla} $\minhash$ requires a fresh $(d+1)$ dimensional permutation to compute  the same. 
\end{remark}
 
We first give a proof of  Theorem~\ref{thm:lift} using Propositions~\ref{prop:indicator}, \ref{prop:high_prob}, \ref{prop:prop7}, \ref{prop:prop_b_zero} and \ref{prop:prop_b_one}.  

\begin{restatable}{prop}{liftSingleInsertFirstProp}\label{prop:indicator}
In Algorithm~\ref{alg:lift}, we have the following: if $i < r$, then $\pi'(i) = \pi(i) + \mathbbm{1}_{\{\pi(i) \geq \pi(r)\}}$;
if $i = r$, then $\pi'(r) = \pi(r)$; if $i > r$, then $\pi'(i) = \pi(i-1) + \mathbbm{1}_{\{\pi(i-1) \geq \pi(r)\}}$.
\end{restatable}
 
\begin{proof}
Note that Algorithm~\ref{alg:lift} initially sets all $\pi'(i) = \pi(i)$ if $i  \leq r$, and $\pi'(i) = \pi(i-1)$ if $i > r$.
If $i < r$, then only if the condition in line~\ref{alg:line11liftPerm}  of Algorithm ~\ref{alg:lift} is satisfied we increment $\pi'(i)$ by $1$, which happens when $\mathbbm{1}_{\{\pi'(i) \geq \pi(r)\}}=1$. 
 If $i = r$, then $\pi'(r)$ is never updated by Algorithm~\ref{alg:lift}.
If $i > r$, then $\pi'(i)$ is initialized in line~\ref{alg:line7liftPerm}  of the algorithm to $\pi(i-1)$. Then  it gets updated in line~\ref{alg:line12liftPerm}  only when the condition in line \ref{alg:line11liftPerm}  is satisfied which happens when 
$\mathbbm{1}_{\{\pi(i) \geq \pi(r)\}}=1.$
\end{proof}

\begin{restatable}{prop}{liftSingleInsertSecondProp}\label{prop:high_prob}
If $X\in \{0, 1\}^d$ with $|X|\leq k$  and $r$ is chosen  uniformly at random from $\{1, 2, , \ldots, d\}$, then with probability at least $1 - O(k/d)$ we have:
$x_r = 0$ and $x_{r+1} = 0.$
\end{restatable}
\begin{proof}
Let $J = \{j : x_j = 1\}$ be the set of non-zero indices of $X$.
Since $r$ is chosen uniformly at random from $\{1,\ldots, d+1\}$, probability that $r-1, r, r+1 \in J$ is at most $O(k/d)$. Thus $x_r = 0$ and $x_{r+1} = 0$ with probability at least $1 - O(k/d)$.
\end{proof}

\begin{restatable}{prop}{liftSingleInsertThirdProp}\label{prop:prop7}
Let $J$ be the set of non-zero indices of $X$, and $J'$ be the set of non-zero indices of $X'$.
Assume that $r-1, r, r+1 \notin J$.
If $b=0$, then there is a bijection from $J$ to $J'$ defined as following mapping from $j \in J$ to $j' \in J'$: if $j<r$, then $j'=j$ and $j> r$, then $j'=j+1$. If $b=1$ then the same bijection holds from $J \cup \{r\} $ to $J'$ by additionally mapping $r$ to $r.$
\end{restatable}
\begin{proof}
If $b = 0$, then from Proposition~\ref{prop:indicator}, we get a bijection from $J$ to $J'$ defined as follows:
if $j < r$, then $j' = j$; if $j > r$ then $j' = j + 1$.
If $b = 1$, then again from Proposition~\ref{prop:indicator} and the fact that $r-1, r, r+1 \notin J$, we get a bijection
from $J \cup \{r\}$ to $J'$ defined as follows:
if $j < r$ then $j' = j;$ if $j = r$ then $j'= j$; if $j > r$ then $j' = j + 1$.

Note that if $r \in J$ then $J \cup \{r\} = J$. Since $|J| = k$ and  $|J'| = k+1$, a bijection between $J \cup \{r\}$ and $J$ would not be possible.
\end{proof}

\begin{restatable}{prop}{liftSingleInsertFourthProp}\label{prop:prop_b_zero}
Let $j_{min} \in J$ be the index such that $\pi(j_{min})$ is minimum among $\{\pi(j) : j \in J\}$. 
If $b = 0$,   then 
$\pi'(j'_{min})$ is  minimum among all $\{\pi'(j') : j' \in J'\}$, where
$j'_{min}$ is obtained from $j_{min}$ via the bijection from Proposition~\ref{prop:prop7}.
\end{restatable}
\begin{proof}
 Let $\pi(j_{min}) = \min\{\pi(j) : j \in J\}$. Recall that we have a bijection from $J$ to $J'$ given by $j \mapsto j'$, where $j\in J$ and $j'\in J'$. We want to show that $\pi'(j'_{min}) = \min \{\pi'(j') : j' \in J'\}$.

 Recall from Proposition~\ref{prop:indicator} that for any $j' \in J'$, we have: $\pi'(j') = \pi(j) + \mathbbm{1}_{\{\pi(j) \geq \pi(r)\}}$. 
 In particular, we have $\pi'(j'_{min}) = \pi(j_{min}) + \mathbbm{1}_{\{\pi(j_{min}) \geq \pi(r)\}}$. 
 Also note that $\pi'(j')$ is either $\pi(j)$ or $\pi(j) + 1$.

 If $\mathbbm{1}_{\{\pi(j_{min}) \geq \pi(r)\}} = 0$, then $\pi'(j'_{min}) = \pi(j_{min})$. 
 Since for any other $j', \pi'(j')$ is either $\pi(j)$ or $\pi(j) + 1$, $\pi'(j'_{min})$ still remains the minimum.
 
 If $\mathbbm{1}_{\{\pi(j_{min}) \geq \pi(r)\}} = 1$ then $\pi'(j'_{min}) = \pi(j_{min}) + 1.$ Moreover, for any other $j' \in J'$, since $\pi(j) \geq \pi(j_{min}) \geq \pi(r)$, the indicator
 $\mathbbm{1}_{\{\pi(j) \geq \pi(r)\}} = 1$ 
 holds. Hence $\pi'(j') = \pi(j) + 1$ also holds.
 Thus we have $\pi'(j'_{min}) = \min\{ \pi'(j'): j' \in J'\}$.
\end{proof}

\begin{restatable}{claim}{liftSingleInsertclaim}\label{}
 With probability at least $1 - O(k/d)$ we have: 
 $r-1, r, r+1 \notin J.$
\end{restatable}
\begin{restatable}{prop}{liftSingleInsertFifthProp}\label{prop:prop_b_one}
Let  $j_{min} \in J \cup \{r\}$ be the index such that $\pi(j_{min})$ is minimum among $\{\pi(j) : j \in J \cup \{r\}\}$.
 If $b = 1$, and  $r-1, r, r+1 \notin J$, then we have that
$\pi'(j'_{min})$ is  minimum among all $\{\pi'(j') : j' \in J'\}$, where
$j'_{min}$ is obtained from $j_{min}$ via the bijection from Proposition~\ref{prop:prop7}.
\end{restatable}
\begin{proof}
  Let $\pi(j_{min}) = \min\{\pi(j) : j \in J \cup \{r\} \}$. 
  Therefore, we have a bijection from $J \cup \{r\}$ to $J'$ given by $j \mapsto j'$, where $j\in J$ and $j'\in J'$. We want to show that $\pi'(j'_{min}) = \min \{\pi'(j') : j' \in J'\}$.
 Recall from Proposition~\ref{prop:indicator} that for any $j' \in J'$, we have: $\pi'(j') = \pi(j) + \mathbbm{1}_{\{\pi(j) \geq \pi(r)\}}$ when $j' \neq r$ and $\pi'(r) = \pi(r)$. 
 Note that $\pi'(j')$ is either $\pi(j)$ or $\pi(j) + 1$.
If $j_{min} = r$, then since $\pi'(r) = \pi(r)$ we have $j'_{min} = r$.

If $j_{min} \neq r$, then  we have $\pi'(j'_{min}) = \pi(j_{min}) + \mathbbm{1}_{\{\pi(j_{min}) \geq \pi(r)\}}$. 

 If $\mathbbm{1}_{\{\pi(j_{min}) \geq \pi(r)\}} = 0$, then $\pi'(j'_{min}) = \pi(j_{min})$. 
 Since for any other $j', \pi'(j')$ is either $\pi(j)$ or $\pi(j) + 1$, $\pi'(j'_{min})$ still remains the minimum.
 
 If $\mathbbm{1}_{\{\pi(j_{min}) \geq \pi(r)\}} = 1$, then $\pi'(j'_{min}) = \pi(j_{min}) + 1.$ Moreover, for any other $j' \in J'$, since $\pi(j) \geq \pi(j_{min}) \geq \pi(r)$, the indicator
 $\mathbbm{1}_{\{\pi(j) \geq \pi(r)\}} = 1$ 
 holds. Hence $\pi'(j') = \pi(j) + 1$ also holds.
 Thus, we have $\pi'(j'_{min}) = \min\{ \pi'(j'): j' \in J'\}$.
\end{proof}

We complete a proof of Theorem~\ref{thm:lift} as follows.\\ 

\noindent \textbf{Proof of Theorem~\ref{thm:lift}:}
\begin{proof} We split in the following cases:\\
\noindent\textbf{ Case 1:} when  $b=0$.
 Let $\pi(j_{min}) = \min\{\pi(j) : j \in J\}$.
 From Proposition~\ref{prop:prop_b_zero}, we know that $\pi'(j'_{min}) = \min\{\pi'(j') : j' \in J'\}$. 
 Since $\pi$ is minwise independent, we have $j_{min}$ is uniformly distributed across $J$. Since $j \mapsto j'$ is a bijection, we have that $j'_{min}$ is uniformly distributed across $J'$. Hence minwise independence for $\pi'$ holds.
 
\noindent\textbf{Case 2:} when $b = 1$.
  Let $\pi(j_{min}) = \min\{\pi(j) : j \in J \cup \{r\}\}$.
 Since $r$ is chosen uniformly random from $\{1,\ldots, d\}$, from Proposition~\ref{prop:high_prob}, with probability at least $1- O(k/d)$,  we have: $x_r = 0$ and $x_{r+1}=0$, i.e., $r \notin J$.
 Recall that we have a bijection from $J \cup \{r\} \mapsto J'$.
 
 Let $\pi(J \cup \{r\}) = \{\pi(j) : j \in J \cup \{r\}\}$.
 Since $x_{r+1} = 0$, we have that $\pi(J \cup \{r\})$ consists of  $\pi(j_1), \pi(j_2), ..., \pi(j_{k+1})$ for $k+1$ distinct indices. Also $\pi'(J')$ consists' of $\pi'(j_1'), \pi'(j_2') ..., \pi'(j_{k+1}')$ where $\pi(j')$ is either $\pi(j)$ or $\pi(j) + 1$. for $k+1$ distinct values of $j$.

 Note that if $x_{r+1} = 1$, then $\pi'(r+1) = \pi(r) + 1$ and $\pi'(r) = \pi(r)$. So the distinctness of values of $j$ would not hold. However with probability at least $1 - O(k/d)$ we have: $x_{r+1} = 0$.

 Since $\pi$ is minwise independent, we have that the minimum index $j_{min}$ is uniformly distributed across 
 $J \cup \{r\}$. Hence from Proposition~\ref{prop:prop_b_one}, we can conclude that $j'_{min}$ is uniformly distributed across $J'$. Hence $\pi'$ is minwise independent.
 \end{proof}


We now give a proof of Theorem~\ref{thm:update_sketch}. To do so, we require the following Propositions~\ref{prop:prop1}, \ref{prop:prop2}, \ref{prop:prop3}. We divide the proof into two main cases based on whether the permutation value of the inserted index is greater than the hash value or not.
The case when the permutation value is less than or equal to the hash value is further divided into two cases based on whether the inserted bit is $0$ or $1.$
\begin{prop}\label{prop:prop1}
If  $ h_{old}<\pi(m) $, then $h_{new}=h_{old}=\mathrm{minHash}_{\pi'_{m}}(X')$.
\end{prop}

\begin{proof}
Note that if $h_{old} < \pi(m)$, then Algorithm~\ref{alg:update_sketch} outputs $h_{new} = h_{old}.$
Let $j_{min}$ be the index such that  $\pi(j_{min}) = h_{old}.$ 
From Proposition~\ref{prop:indicator}, we have: 
$\pi'_{m}(j'_{min}) = \pi(j_{min}) + \mathbbm{1}_{\{\pi(j_{min}) \geq \pi(m)\}}$.

Since $\pi(j_{min}) = h_{old} < \pi(m)$, we have that $\mathbbm{1}_{\{\pi(j_{min}) \geq \pi(m)\}}=0$.
Hence we have $\pi'_{m}(j'_{min}) = \pi(j_{min}) = h_{old}$.
Note that since $\pi'(m) = \pi(m) > \pi'_{m}(j_{min})$ we can conclude $\pi'(m)$ can not be new minimum.
From Proposition~\ref{prop:prop_b_zero} and \ref{prop:prop_b_one}, we have $\pi'_{m}(j'_{min}) = \mathrm{minHash}_{\pi'_{m}}(X')$.
\end{proof}
\begin{prop}\label{prop:prop2}
 If $ h_{old} \geq \pi(m)$ and $b = 1$, then
 $h_{new} = \pi(m)=\mathrm{minHash}_{\pi'_{m}}(X')$.
\end{prop}

\begin{proof}
If $h_{old} = \pi(m)$ then Algorithm~\ref{alg:update_sketch} outputs $h_{new} = \pi(m)$ in line~\ref{alg:line8liftDrop}.
Also note that from Proposition~\ref{prop:indicator}, we have for any $j' \in J', \pi'_{m}(j') \geq \pi(j) \geq \pi(m).$ Hence, $\mathrm{minHash}_{\pi'_m}(X')=\pi(m)=h_{new}.$

If $h_{old} > \pi(m)$ then
$\pi(m) = \min \{ \pi(j) : j \in J \cup \{m\}\}$. Therefore, from Proposition~\ref{prop:prop_b_one}, we have
$\mathrm{minHash}_{\pi'_{m}}(X') = \pi'_{m}(m) = \pi(m).$
From line~\ref{alg:line8liftDrop} of Algorithm~\ref{alg:update_sketch}, we have $h_{new} = \pi(m)=\mathrm{minHash}_{\pi'_{m}}(X').$
\end{proof}

\begin{prop}\label{prop:prop3}
If $ h_{old}\geq \pi(m)$ and $b = 0$, then
 $h_{new} = h_{old}+1=\mathrm{minHash}_{\pi'_{m}}(X')$.
\end{prop}

\begin{proof}
Since $b=0$, note that $\pi'_m(m)$ is not a candidate for the $\mathrm{minHash}_{\pi'_{m}}(X').$
From Proposition~\ref{prop:indicator}, we know that
$\pi'_{m}(j') = \pi(j) +    \mathbbm{1}_{\{\pi(j) \geq \pi(m)\}}$ for any $j \neq m.$
Since $h_{old} \geq \pi(m)$ we know that for any $j \in J,$ we have  $\mathbbm{1}_{\{\pi(j) \geq \pi(m)\}}=1$. 
Hence $\min \{\pi'_{m}(j') : j' \in J'\} = \min \{\pi(j) + 1: j \in J\} = h_{old} + 1.$
Also note that Algorithm~\ref{alg:update_sketch}, in line~\ref{alg:line11liftDrop}, outputs, $h_{new} = h_{old} + 1$, if $ h_{old}\geq \pi(m)$ and $b = 0$. Hence, we have  $h_{new} = h_{old}+1=\mathrm{minHash}_{\pi'_{m}}(X')$.
\end{proof}

\noindent\textbf{Proof of Theorem}~\ref{thm:update_sketch}:
Propositions~\ref{prop:prop1}, \ref{prop:prop2}, and  \ref{prop:prop3} completes a proof of Theorem~\ref{thm:update_sketch}.

\begin{remark}
We can extend our results for multiple feature insertion by repeatedly applying  Theorem~\ref{thm:lift}, and Theorem~\ref{thm:update_sketch} along with the probability union bound. However, the  time complexity of the algorithm obtained by sequentially inserting $n$ features will grow linearly in $n$ as observed in the empirical results (Figure~\ref{fig:feature_insertion_exp_plot}, Section~\ref{sec:experiments}). In the next subsection, we present an algorithm that performs multiple insertions in parallel, which helps us achieve much better speedups.
\end{remark}

  \subsection{Algorithm for multiple feature insertions  -- $\mathrm{multipleLiftHash}$:}\label{subsec:multipleLiftHash}
  
    \begin{table*}[!ht]
       \caption{\footnotesize{Notations}\label{tab:notations_multiple_insertion}} \centering
   \scalebox{0.85}{
     \noindent\begin{tabular}{lc||lc}
    \hline
    No. of inserted features & $n$ & Position of inserted features $\{m_i\}_{i=1}^n$, $m_i\in[d+1]$  & $M$ \\
    \hline
    $X$ after $n$ features insertion $\{0,1\}^{d+n}$& $X'$ & Set of inserted bits $\{b_1,\ldots, b_n\}$ with $b_i\in\{0,1\}$  & $B$ \\
    \hline
     $\mathrm{multipleLiftHash}(M,\pi,B,h_{old})$& $h_{new}$     & Lifted $(d+n)$-dim. permutation    & $\pi'_M$\\
     \hline
    \end{tabular}
}
\end{table*}
 Results presented in this subsection are extensions to that of Subsection~\ref{subsec:liftHash}. The intuition of our proposal is that we can (implicitly) generate a new $(d+n)$-dimensional permutation ($n$ is the number of inserted features), using the old $d$-dimensional permutation. By exploiting the sparsity of input, and the fact that inserted bits are random positions, we show that the updated permutation satisfies the min-wise independent property with high probability. Further, we suggest a simple update rule aggregating the existing $\minhash$ sketch and the $\minhash$ restricted to inserted position and outputs the updated sketch.

\begin{algorithm}[H]
\DontPrintSemicolon
\SetNoFillComment
 \textbf{Input:~}{Permutation \(\pi\), a sorted set of indices   \(M=\{m_1,...m_n\}\), and set of inserted bits $B = \{b_1, \ldots, b_n\}$}\\
 \textbf{Output:~}{The min value of $\pi$ (with appropriate shift) restricted to only those indices $m_i$ of $M$ that correspond to non-zero $b_i$. }\\
 {\color{black}$\pi_{M,B}$ = \{$\pi(m_i) \mid i  \in \{1, \ldots, n\} \text{~and~} b_i=1$\}}\\
 
  \Return {\color{black}$\min$\{$\pi_{M,B}(k) + 1$\} }
\caption{$\mathrm{partialMinHash}(\pi, M, B)$ }\label{alg:partial_minhash}
\end{algorithm}

\begin{algorithm}[H] 
\DontPrintSemicolon
\SetNoFillComment
 \textbf{Input:~}{Permutation $\pi$; $R$ with $|R|=n$.}\\
 \textbf{Output:~}{$(d+n)$-dim. permutation $\pi'$.}\\
 \(R \leftarrow \mathrm{sorted}(R)\) \tcc{sorting array $R$ in ascending order}
\For{$i \in \{1,2,\ldots n\}$}
  {
   $R[i] = R[i]+i-1$\\
   }
\(\pi' = \pi\)~~~~ \tcc{Initialization}
\For{$i \in \{1,\ldots n\}$}
  {
   $\pi' = \mathrm{liftPerm}(\pi',R[i])$ \tcc{Calling \text{Algorithm}~\ref{alg:lift} with $\pi=\pi'$ and $r=R[i]$}
   }
\Return ${\pi}’$\\
\caption{$\mathrm{multipleLiftPerm}(\pi,R)$.}\label{alg:lift2}
\end{algorithm}

\begin{algorithm}[H] 
\DontPrintSemicolon
\SetNoFillComment
 \textbf{Input:~}{ $h_{old}:=\mathrm{minHash}_{\pi}(X)$,   permutation $\pi$, $M$ and $B$.} \\
 \textbf{Output:~}{$h_{new}:=\mathrm{multipleLiftHash}(M,\pi,B,h_{old})$.}\\%
Let $\pi_M := \{\pi(m) : m \in M\}$.\\
{ $a_M = \mathrm{partialMinHash}(\pi,M,B)$}\\
    $h_{new}= \min\left(h_{old} + |\{x \mid x \in \pi_M \text{~and~} x \leq h_{old}\}|, a_M \right)$
 \tcc{Picking the minimum between $\mathrm{partialMinHash}$ and shifted value of $h_{old}$. }
 \Return $h_{new}$
\caption{$\mathrm{multipleLiftHash}(M,\pi,B,h_{old})$.}\label{alg:update_sketch_multiple_insertion}
\end{algorithm}

Algorithm~\ref{alg:update_sketch_multiple_insertion} takes $h_{old}$,  $M,B$, and $\pi$ as input, and outputs the updated sketch $h_{new}$. Algorithm~\ref{alg:update_sketch_multiple_insertion} uses Algorithm~\ref{alg:partial_minhash} to obtain the value of $\mathrm{partialMinHash}$ -- minimum $\pi$ value restricted to the inserted indices only with inserted bit value $1$, from which it obtains $\mathrm{multipleLiftHash}$ for the updated input. Algorithm~\ref{alg:lift2} is implicit and is used to prove the correctness of Algorithm~\ref{alg:update_sketch_multiple_insertion}. Algorithm~\ref{alg:lift2} takes the permutation $\pi$ and   $M$ as input, and outputs a $(d+n)$-dimensional  permutation $\pi’_M$ which satisfies the condition stated in Equation~\eqref{eq:minwise} for  $X$, with $|X|\leq k$. We show this in Theorem~\ref{thm:lift2}. Then in Theorem~\ref{thm:update_sketch_multiple}, we show that $h_{new}=\minhash_{\pi’_M}(X’)$. As $\pi’_M$ satisfies the condition stated in Equation~\eqref{eq:minwise} for sparse $X$, then  due to Equation~\eqref{eq:minwise_jaccard} and \citep{BroderCFM98} the sketch of data points obtained from Algorithm~\ref{alg:update_sketch_multiple_insertion} approximates the Jaccard similarity.


\begin{example}
 Suppose \(X = [1,0,0,1,0,1,0]\) and \(\pi = [6,3,1,7,2,5,4]\) are input point and original permutation, respectively. Then the value of  $h_{old}$ is $5$. Let $M=[2,4]$ and \(B = [0,1]\). Thus, in this case \(\pi'_M = [7,3,4,1,8,9,2,6,5]\) and \(X' =[1,0,0,0,1,1,0,1,0].\) Consequently we have,  $\mathrm{partialMinHash}(\pi,M,B) = 2  < h_{old}+|\{x \mid x \in \pi_M \text{~and~} x \leq h_{old}\}|= 5+1=6.$ Therefore, \(\mathrm{minHash}_{\pi'_M}(X') = 2\).
\end{example}

We have the following theorems for the correctness of the algorithms presented in this subsection. A proof of the Theorem~\ref{thm:lift2} follows similarly to the proof of Theorem~\ref{thm:lift} along with the probability union bound,  and the proof of Theorem~\ref{thm:update_sketch_multiple} is a generalization of proof of Theorem~\ref{thm:update_sketch}. 
 \begin{theorem}\label{thm:lift2}
 Let  $\pi$ be a minwise independent permutation.
 Let $M = \{m_1, \ldots, m_n\}$ such that $m_i$ is chosen uniformly at random from $\{1, \ldots, d\}.$
 Then for any $X\in \{0,1\}^d$ with $|X|\leq k$, the permutation $\pi’_M $ obtained from Algorithm~\ref{alg:lift2} satisfies the condition stated in Equation~\eqref{eq:minwise}   of Definition~\ref{def:minwise}, with probability $1 - O(kn/d).$
 
 \end{theorem}

 \begin{restatable}{theorem}{insertMultipleDelUpdateSketch}\label{thm:update_sketch_multiple}
 Let $\pi’_{M}$ be the $(d+n)$-dimensional permutation outputted  by Algorithm~\ref{alg:lift2}, if we set  \(R=M\). Then, the  sketch obtained from Algorithm~\ref{alg:update_sketch_multiple_insertion} is exactly the same as the  sketch  obtained with  the permutation ${\pi}’_M$ on $X'$, that is, $\mathrm{multipleLiftHash}(\pi, M,B,h_{old}) = \mathrm{minHash}_{\pi'_M}(X')$.
 \end{restatable}

We require the following propositions in order to prove the Theorem~\ref{thm:update_sketch_multiple}. 
\begin{prop}\label{prop:prop999}
The bit whose index is \(h_{old}\) in \(X_{\pi}\) has index  \(h_{old}+|\{x \mid x \in \pi_M \text{~and~} x \leq h_{old}\}|\) in \(X'^{\pi'_M}\).
\end{prop}

\begin{proof}
While extending from \(X^{\pi}\) to \(X'^{\pi'_M}\) we know that we have added the bits \(B\) at positions \(M\). So the index in \(X^{\pi}\) where \(h_{old}\) existed has been shifted by a number of units to the right  to form \(X'^{\pi'_M}\). The number of units it has been shifted will be equal to the number of bits in \(\pi'_{M}\) that is less than \(h_{old}\), and hence the result follows.

\end{proof}

\begin{prop}\label{prop:prop888} The
\(\mathrm{partialMinHash}(\pi,M,B)\) represents the $\minhash$ value of \(X'^{\pi'_M}\) restricted to only non-zero indices of $B$, if only the elements at newly added indices \(M\) were taken into account, i.e. after shifting the corresponding $\pi$ values by the number of insertions.
\end{prop}

\begin{proof}

We take into account the number of insertions that have happened before the insertion at index \(i\) for \(i  \in \{1,\ldots,n\}\) and determine the \(\pi'_M\) value for each newly added bit. Then, we 
calculate the minimum among all such \(\pi'_M\) values where a $1$ has been inserted in \(X'\)which in other words is $\minhash$ in terms of just the inserted indices. 

\end{proof}

\begin{prop}\label{prop:prop7676}
The $\minhash$ value of \(X'\) with respect to the permutation \(\pi'_M\) is given by the minimum of \(\mathrm{partialMinhash}(\pi,M,B)\) and  \(h_{old}+|\{x \mid x \in \pi_M \text{~and~} x \leq h_{old}\}|\).
\end{prop}
\begin{proof}
We have proved that  \(\mathrm{partialMinhash}(\pi,M,B)\) returns the $\minhash$ value with respect to only the newly added indices (assuming bits at old indices are all $0$). And if the newly added elements were assumed to be $0$ then the $\minhash$  will be \(h_{old}+|\{x \mid x \in \pi_M \text{~and~} x \leq h_{old}\}|\) (due to the shift we showed earlier).

Now $\minhash$  is the first time we see $1$ while traversing through the indices of \(X'^{\pi'_M}\) from left to right. So the first time $1$ will occur will either happen in the new indices or the old indices. If it occurs in the old indices, then it's bound to be at  \(h_{old}+|\{x \mid x \in \pi_M \text{~and~} x \leq h_{old}\}|\) and if it happens at the new indices it will happen at \(\mathrm{partialMinhash}(\pi,M,B)\). So the first time it occurs will be at the minimum of the two values.

\end{proof}

\noindent\textbf{Proof of Theorem~\ref{thm:update_sketch_multiple}:}
 Propositions~\ref{prop:prop999},~\ref{prop:prop888},~\ref{prop:prop7676} completes a proof of the theorem.

 \section{Algorithm for feature deletion}\label{sec:feature_deletion}
We first give our result for one feature deletion. 
\subsection{One feature deletion at a time -- $\mathrm{dropHash}:$}\label{subsec:dropHash}
  We denote $X’ \in\{0, 1\}^{d-1}$ as the data point after one feature deletion. The intuition of our algorithm is that we can (implicitly) generate a new $(d-1)$-dimensional permutation by creating a bijection between the input indices before and after feature deletion. This preserves the distribution of the minimum index with respect to permutation, and ensures the minwise independent property stated in Equation~\eqref{eq:minwise}. We discuss this in Algorithm~\ref{alg:lift_del} that takes   $\pi$, and $m$ as input, and outputs a $(d-1)$ dimensional  permutation $\pi’_m$. We give its proof of correctness in Theorem~\ref{thm:lift_del}, where we show that $\pi'_m$ satisfies the minwise independent property stated in Equation~\eqref{eq:minwise}.
 Further, the corresponding sketch updation \textit{w.r.t.} the new permutation is done via a simple update rule mentioned in Algorithm~\ref{alg:update_sketch_del}. The algorithm takes $h_{old}$, the position of the deleted feature $m$, and the corresponding value $b$  as input, and outputs the updated sketch $h_{new}$.  We give a proof of correctness of Algorithm~\ref{alg:update_sketch_del} in Theorem~\ref{thm:update_sketch_del}, where we show that $h_{new}=\minhash_{\pi’_m}(X’)$.
Therefore, due to Theorems~\ref{thm:lift_del},~\ref{thm:update_sketch_del}, and   Equation~\eqref{eq:minwise_jaccard} (and \citep{BroderCFM98}) the sketch of data points obtained after Algorithm~\ref{alg:update_sketch_del} approximates the pairwise Jaccard similarity. We illustrate our algorithm with the following  example, and then we state its proof of correctness in Theorems~\ref{thm:lift_del}, \ref{thm:update_sketch_del}.

  \begin{algorithm}[H] 
\DontPrintSemicolon
 \textbf{Input:~}{ $d$-dimensional permutation $\pi$, index position $r$.}\\
 \textbf{Output:~}{ $(d-1)$-dimensional permutation ${\pi}'$.}  \\
{
   \For{ $i \in \{1,\ldots, d-1\}$}
   {
   \eIf{$i<r$}
   {$\pi’(i) = \pi(i)$}
   {    $\pi’(i) = \pi(i+1)$}
   }
   \For{ $i \in \{1,\ldots, d-1\}$}
   {
   \If{${\pi}’(i) >  {\pi}(r)$ \label{alg:line11dropPerm} }
   {
   ${\pi}’(i) = {\pi}’(i) - 1$ \label{alg:line12dropPerm}
   }
   }
   \Return $\pi'$
}
\caption{$\mathrm{dropPerm}(\pi, r)$.}\label{alg:lift_del}
\end{algorithm}

\begin{algorithm}[H] 
\DontPrintSemicolon
\SetNoFillComment
  \textbf{Input:~}{ $h_{old}$, $\pi$, $X$, $m \in [d]$, and $b\in \{0, 1\}$.}\\
 \textbf{Output:~}{$h_{new}:=\mathrm{dropHash}(m,X,\pi,h_{old})$.}\\%
 Let $X^\pi[i] := X[j]$, \textit{s.t.} $\pi(j) = i$ and $a_m:=\pi(m)$\\
 \tcc{$X^\pi$ is values of $X$ permuted according to $\pi$. }
 \eIf{$h_{old} < a_m$}
 {$h_{new} = h_{old}$}{
 \If{$h_{old} > a_m$}{$h_{new}  = h_{old} - 1$}
 \If{$h_{old} = a_m$}{
 \label{alg:line10dropHash}
 \tcc{Recalculate $\minhash$ from scratch.}
  \For{ $i \in \{a_{m}+1,....d\} $ \label{alg:line11dropHash}}
   {
   \If{$X^{\pi}[i] = 1$}
   {
   $h_{new} = i-1$ and \texttt{exit;} \label{alg:line13dropHash}
   }
   }
   }
}
\Return $h_{new}$ \label{alg:line18dropHash}
\caption{$\mathrm{dropHash}(m,X,\pi,h_{old})$.}\label{alg:update_sketch_del}
\end{algorithm}

\begin{example}
Let  \(X= [1,0,0,1,0,1,0]\) be the input, and \(\pi= [6,2,1,7,3,5,4]\) be the original permutation. The value of $h_{old}=5$. Suppose that we delete the feature at the index $5$. Then \(X'=[1,0,0,1,1,0]\) and due to Algorithm~\ref{alg:lift_del} the value of  \(\pi'_m=[5,2,1,6,4,3]\). We have $\pi(m)=a_m=3$. We calculate the value of $h_{new}$ outputted by Algorithm~\ref{alg:update_sketch_del}: as \( h_{old}=5 > a_m = 3\),   we  have  \(h_{new} = h_{old} -1 = 5-1 =4\). Further, $\minhash_{\pi'_m}(X')=4$. Therefore, we have $h_{new}=\minhash_{\pi'_m}(X')$.
\end{example}

 \begin{restatable}{theorem}{liftSingleDel}
\label{thm:lift_del}
  Let  $\pi$ be a $d$-dimensional minwise independent permutation. Then, for any index position $r \in \{1, \ldots d\}$,   the $(d-1)$-dimensional permutation $\pi’$  obtained from Algorithm~\ref{alg:lift_del} satisfies the condition of being minwise independent permutation mentioned in Definition~\ref{def:minwise}).
 \end{restatable}
\begin{proof}
Let $J \subseteq \{1, 2, \ldots, d\}$ denote the set of the non-zero indices of $d$ dimensional binary vector $X$,
and let $J' \subseteq \{1,2,\ldots, d-1\}$  denote the set of non-zero indices of $d-1$ dimensional binary vector $X'$ referred in Algorithm~\ref{alg:update_sketch_del}. 

Consider the bijection from $J - \{r\}$ to $J'$ defined as follows:
$$j \mapsto j', ~\text{where~} j' = j + \mathbbm{1}_{\{j \geq r\}}.$$
Note that  from line \ref{alg:line11dropPerm} and \ref{alg:line12dropPerm}  of Algorithm~\ref{alg:lift_del}, we have:
\begin{align*}
\pi'(j') &= \pi(j) - \mathbbm{1}_{\{\pi(j) > \pi(r)\}}. \numberthis\label{eq:eq111}
\end{align*}

We want to show that $\min (\pi'(J')) := \min \{\pi'(j') : j' \in J'\}$ is uniformly distributed across $J'$.
Let $J'_{\pi(j) > \pi(r)} := \{j' : \mathbbm{1}_{\{\pi(j) > \pi(r)\}} = 1\}$
and $J'_{\pi(j) < \pi(r)} := \{j' : \mathbbm{1}_{\{\pi(j) < \pi(r)\}}=1 \}$. 
Note that since we have removed $x_r$, we do not have $\pi(j) = \pi(r)$ corresponding to any $j'$.
Let $J_{\pi(j) > \pi(r)} := \{j : \mathbbm{1}_{\{\pi(j) > \pi(r)\}} = 1\}$
and $J_{\pi(j) < \pi(r)} := \{j : \mathbbm{1}_{\{\pi(j) < \pi(r)\}}=1\}$.

Now $$\min (\pi'(J')) = \min \{\min(\pi'(J'_{\pi(j) > \pi(r)}), \min(\pi'(J'_{\pi(j) < \pi(r)})\}.$$

From Equation~\ref{eq:eq111}, we have
\begin{align*}
\min(\pi'(J'_{\pi(j) < \pi(r)})) &= \min \{\pi(j) : j' \in J'_{\pi(j) < \pi(r)}\}.\numberthis\label{eq:eq2222}\\
 \min(\pi'(J'_{\pi(j) > \pi(r)})) &= \min \{\pi(j) - 1: j' \in J'_{\pi(j) > \pi(r)}\}.\numberthis\label{eq:eq3333}
\end{align*}
Since $\pi$ is minwise independent, RHS in Equation~\ref{eq:eq2222}  is uniformly distributed across the set. Also, RHS in Equation~\ref{eq:eq3333} is uniformly distributed across the set.

Moreover, since $J'_{\pi(j) < \pi(r)}$ and $J'_{\pi(j) > \pi(r)}$ are disjoint
and $\{j: j' \in J'_{\pi(j) < \pi(r)}\}$ and $\{j : j' \in J'_{\pi(j) > \pi(r)}\}$ are also disjoint,
by minwise independent of $\pi$ we can conclude that $\min(\pi'(J'))$ is uniformly distributed across $J'$.
\end{proof}

 \begin{restatable}{theorem}{liftSingleDelUpdateSketch}\label{thm:update_sketch_del}
 Let $\pi’_{m}$ be the $(d-1)$-dimensional permutation outputted  by Algorithm~\ref{alg:lift_del} by setting $r=m$. Then, the  sketch obtained from Algorithm~\ref{alg:update_sketch_del} is exactly the same as the  sketch  obtained via  the permutation ${\pi}’_m$ on $X'$, that is, $\mathrm{dropHash}(m,X,\pi,h_{old}) = \mathrm{minHash}_{\pi'_m}(X')$.
 \end{restatable}
 
 We require  Propositions~\ref{prop:prop444}, \ref{prop:prop445}, \ref{prop:prop446} to prove the theorem. 
\begin{prop}\label{prop:prop444}
If $h_{old} < a_m$    then $h_{new} = h_{old}=\mathrm{minHash}_{\pi'_{m}}(X')$.
\end{prop}

\begin{proof}
We know that $h_{old}$ is the minimum index at which we see a $1$ while iterating  through the features of $X$ in order permutation of $\pi$, which is the  same as going through the elements of \(X^{\pi}\) in the order \(\{1,...,d\}\). We   delete \(a_m = \pi(m)\) at the index $m$ of $X$ which occurs after index \(h_{old}\) of $X^\pi$ . 
Here $X^{\pi}$ and $X'^{\pi'_m}$  will look as follows:
\begin{align*}
    X^\pi &= (0,\ldots,  X^\pi[h_{old} - 1] = 0, X^\pi[h_{old}]=1, \ldots, X^\pi[d]), \qquad \text{and}\\
    X'^{\pi'_{m}} &= (0,\ldots,  X^{\pi}[h_{old}] = 1,...,X^{\pi}[a_{m}-1],X^{\pi}[a_{m}+1]
\ldots, X^{\pi}[d]).
\end{align*}
Therefore, the minimum index at which we see a $1$ in $X'^{\pi'_m}$ remains at index \(h_{old}\), and the desired result follows.
\end{proof}

\begin{prop}\label{prop:prop445}
If $h_{old} > a_m$,  then $h_{new} = h_{old}-1=\mathrm{minHash}_{\pi'_{m}}(X')$.
\end{prop}

\begin{proof}
We know that $h_{old}$ is the minimum index at which we see a $1$ while iterating through the features of $X$ in order of permutation $\pi$, which is the same as going through the elements of \(X^{\pi}\) in the order \(\{1,...,d\}\). We have deleted an element  \(a_m = \pi(m)\) at the index $m$ of $X$ which occurs before index \(h_{old}\) of $X^\pi$. Here $X^{\pi}$ and $X'^{\pi'}$  will look as follows:
\begin{align*}
    X^\pi &= (0, \ldots,  X^\pi[h_{old} - 1] = 0, X^\pi[h_{old}]=1, \ldots, X^\pi[d]),\qquad \text{and,}\\
    X'^{\pi'} &= (0,\ldots,X^{\pi}[a_{m}-1],X^{\pi}[a_{m}+1],\ldots,X^{\pi}[h_{old}-1] = 1, \ldots, X^{\pi}[d]).
\end{align*}


Therefore, the minimum index at which we see a $1$ in $X'^{\pi'}$ is at index \(h_{old}-1\). Since index \(a_m\) has been deleted, the desired result follows.
\end{proof}

\begin{prop}\label{prop:prop446}
If $h_{old} = a_m$,  then $h_{new} =\mathrm{minHash}_{\pi'_m}(X')$.
\end{prop}

\begin{proof}
Note that line \ref{alg:line11dropHash}-\label{alg:line13dropHash}  of Algorithm~\ref{alg:update_sketch_del} finds the next index where $1$ occurs, in a brute-force way.
\end{proof}

\noindent\textbf{Proof of Theorem~\ref{thm:update_sketch_del}:}
Propositions~\ref{prop:prop444},~\ref{prop:prop445},~\ref{prop:prop446} completes a proof of the theorem.

\begin{remark}
The only expensive case for Algorithm~\ref{alg:update_sketch_del} happens when $h_{old}=a_m$ considered in line  (\ref{alg:line10dropHash}-\ref{alg:line18dropHash})
In this case, the algorithm has to compute the sketch in a brute-force way. However, as we choose $m$ uniformly at random, this case happens with a probability of $1/d$.
\end{remark}

\begin{remark}
We can extend our results for multiple feature deletion by repeatedly applying  Theorem~\ref{thm:lift_del}, and Theorem~\ref{thm:update_sketch_del} using the probability union bound. However, the  time complexity of this approach 
grows linearly in $n$ as also observed in the empirical results (Figure~\ref{fig:feature_deletion_exp_plot}, Section~\ref{sec:experiments}). In the following subsection, we present an algorithm that performs multiple parallel deletions that helps achieve much better speedups.

\end{remark}

\subsection{Algorithm   for multiple feature deletion  -- $\mathrm{multipleDropHash}$:}\label{sec:multiDropHash}
\begin{algorithm}[H] 
\DontPrintSemicolon
\SetNoFillComment
 \textbf{Input:~}{Permutation $\pi$, array $R$ with $|R|=n$.}\\
 \textbf{Output:~}{$(d-n)$-dimensional  permutation ${\pi}'.$}\\
\(R \leftarrow \mathrm{sorted}(R)\)\\ \tcc{sorting array $R$ in the ascending order}
 \For{ $i=1$ \text{to}  $n$}
  {
  $R'[i] = R[i]-(i-1)$
  }
   
$\pi = \pi'$ \tcc{Initialization Step}
 \For{$i \in \{1,2,\ldots n\}$}
  {
  $\pi' = \mathrm{dropPerm}(\pi',R'[i])$\\ \tcc{Calling \text{Algorithm}~\ref{alg:lift_del} with $\pi=\pi'$ and $r=R'[i]$}
  }
  \Return ${\pi}’$\\
 \caption{$\mathrm{multipleDropPerm}(\pi, R)$.}\label{alg:lift_del_multiple}
\end{algorithm}
\begin{algorithm}[H] 
\DontPrintSemicolon
\SetNoFillComment
 \textbf{Input:~}{ $h_{old}:=\mathrm{minHash}_{\pi}(X)$,  $\pi$,  $M$, $X$.}
 \textbf{Output:~}{$h_{new}:=\mathrm{multipleDropHash}(M,X,\pi,h_{old})$.}\\%
 Let $\pi(M) :=  \{\pi(m) : m \in M\}$\\
  \eIf {$h_{old} < \min(\pi_M)$}
        {$h_{new}=h_{old}$}
  {
     { \eIf{$h_{old} \notin  \pi_M$}
         {$h_{new}=h_{old}- \left|\{x:  x \in \pi_M {~and~} x \leq h_{old}\}\right| $ }
         {\tcc{Compute $\minhash$ from scratch.}~\label{alg:line11multipleDropHash}
         $h_{new}= \min\{\pi(i): x_i=1 \text{~and~} \pi(i) \notin \pi_M\} - |\{x:  x \in \pi_M \text{~and~} x \leq h_{old}\}|$ }
       
     }
     \Return $h_{new}$
}
\caption{$\mathrm{multipleDropHash}(M,X,\pi,h_{old})$}\label{alg:update_sketch_multiple_deletion}
\end{algorithm}

The results presented in this subsection give algorithms for multiple feature deletion and are generalizations of the result presented in Subsection~\ref{subsec:dropHash}. We consider two algorithms for enabling $\minhash$ for multiple feature deletion. Suppose we have a data point $X \in \{0, 1\}^d$ and its $\minhash$ with permutation $\pi$ is $h_{old}$. Let $X’ \in\{0, 1\}^{d-n}$ be the data point after deleting $n$ features. Algorithm~\ref{alg:update_sketch_multiple_deletion} takes $h_{old}$, positions of the  deleted feature \(M = \{m_1,m_2,\ldots,m_n\}\) and outputs the updated sketch $h_{new}$. Algorithm~\ref{alg:lift_del_multiple} is implicit and is used to prove the correctness of Algorithm~\ref{alg:update_sketch_multiple_deletion}. Algorithm~\ref{alg:lift_del_multiple} takes the permutation $\pi$, and $M$ as input, and outputs a $(d-n)$ dimensional  permutation $\pi’_M$. We show in Theorem~\ref{thm:multiple_lift_del}   that  $\pi'_M$ satisfies the condition stated in Equation~\eqref{eq:minwise}. Then in Theorem~\ref{thm:update_sketch_multiple_del}, we show that $h_{new}=\minhash_{\pi’_M}(X’)$. Therefore, due to  Theorems~\ref{thm:multiple_lift_del},~\ref{thm:update_sketch_multiple_del}, and   Equation~\eqref{eq:minwise_jaccard} (and \citep{BroderCFM98}) the sketch of data points obtained after Algorithm~\ref{alg:update_sketch_multiple_deletion} approximates the pairwise Jaccard similarity.

\begin{example}\label{example:muldrophash}

Suppose our input, original permutation, and  the list of deleted features are \(X = [1,0,0,1,0,1,0]\),  \(\pi = [6,3,1,7,2,5,4]\), and  \(M = [2,4]\), respectively. Thus after deletion \(X' = [1,0,0,1,0]\) and \(\pi' = [5,1,2,4,3]\). The value of  $h_{old}=5$. We can calculate \(\pi(M)  = [3,7]\) and therefore \(\min(\pi(M)) = 3\).  Thus, in this case \(h_{old}=5>3=\min(\pi(M))\) and also $h_{old} \notin \pi(M)$. Therefore,  the \(\mathrm{minHash}_{\pi'_M}(X')\) is \(h_{old}-|\{x \mid x \in \pi_M \text{~and~} x \leq h_{old}\}|= 5-1 = 4\).

\end{example}

\begin{theorem}\label{thm:multiple_lift_del}
 Let  $\pi$ be a $d$-dimensional minwise independent permutation.
 Then for every $R$, 
 the $(d-n)$ dimensional permutation $\pi’$ obtained from Algorithm~\ref{alg:lift_del_multiple} satisfies the condition of being minwise independent permutation. 
 \end{theorem}
\begin{proof}
A proof follows by repeated application of  Theorem~\ref{thm:lift_del}.
\end{proof}
 
  \begin{restatable}{theorem}{dropMultipleDelUpdateSketch}\label{thm:update_sketch_multiple_del}
 Let $\pi’_{M}$ be the $(d-n)$-dimensional permutation outputted  by Algorithm~\ref{alg:lift_del_multiple} by setting \(R=M\). Then, the  sketch obtained from Algorithm~\ref{alg:update_sketch_multiple_deletion} is exactly the same as the  sketch  obtained with  the permutation ${\pi}’_M$ on $X'$, that is, $\mathrm{multipleDropHash}(\pi, M,B,h_{old}) = \mathrm{minHash}_{\pi'_M}(X')$.

\end{restatable}
We require the following propositions to prove the theorem.

\begin{prop}\label{prop:prop3333}
If $h_{old} < \min(\pi(M))$,  then $h_{new} = h_{old}=\mathrm{minHash}_{\pi'_M}(X')$.
\end{prop}

\begin{proof}
Recall that  $X^\pi[i] := X[j]$, where $\pi(j) = i$. We know that $h_{old}$ is the minimum index at which we see a $1$ while going through the points of $X$ in order permutation $\pi$, which is the same as going through the elements of \(X^{\pi}\) in the order \(\{1,...,d\}\). Now we have deleted a set of indices $M = \{m_1,\ldots,m_n\}$ which occurs after \(h_{old}.\)   
Here $X_{\pi}$ and $X'_{\pi'_M}$  will look as follows:
\begin{align*}
X^\pi &= (0,\ldots, X^\pi[h_{old} - 1] = 0, X^\pi[h_{old}]=1, \ldots, X^\pi[d]), \qquad \text{and}\\
X'^{\pi'_M} &= (0,\ldots,  X^{\pi}[h_{old}] = 1, \ldots, X^{\pi}[d-n]).
\end{align*}
 So we see that the first time we see a $1$ remains at index \(h_{old}\), and the desired result follows.
\end{proof}

\begin{prop}\label{prop:prop4444}
If $h_{old} \notin \pi(M)$ and $h_{old} > \min (\pi(M))$, then   $h_{new} = h_{old}-|\{x \mid x \in \pi_M \text{~and~} x \leq h_{old}\}|=\mathrm{minHash}_{\pi'_M}(X').$
\end{prop}

\begin{proof}
We know that $h_{old}$ is the minimum index at which we see a $1$ while going through the points of $X$ in order permutation $\pi$, which is the same as going through the elements of \(X^{\pi}\) in the order \(\{1,...,d\}\). Now we have deleted an element at indices $M=\{m_1,\ldots,m_n\}$ which does not contain the index containing \(h_{old}.\)  So in \(X'^{\pi'}\) \(h_{old}\) will be shifted to the left by the number of indices deleted before \(h_{old}\) which is given by \(|\{x \mid x \in \pi_M \text{~and~} x \leq h_{old}\}|\). Therefore,  $X_{\pi}$ and $X'_{\pi'}$  will look as follows:
\begin{align*}
    X^\pi& = (0,\ldots  X^\pi[h_{old} - 1] = 0, X^\pi[h_{old}]=1, \ldots, X^\pi[d]), \qquad \text{and}\\
    X'^{\pi'_M} &= (0, \ldots,X^{\pi}[h_{old}-|\{x \mid x \in \pi_M \text{~and~} x \leq h_{old}\}| = 1,\ldots, X^{\pi}[d-n]).
\end{align*}
 

So we see that the first time we see a $1$ is at index            \(h_{old}-|\{x \mid x \in \pi_M \text{~and~} x \leq h_{old}\}|)\) since index \(a_m\) has been deleted and the desired result follows.

\end{proof}

\begin{prop}\label{prop:prop5555}
If $h_{old}$ is in $\pi({M})$,  then $h_{new} =\mathrm{minHash}_{\pi'_M}(X')$.
\end{prop}
\begin{proof}
Since \(h_{old}\) has been deleted note that line number~\ref{alg:line11multipleDropHash}  of Algorithm~\ref{alg:update_sketch_multiple_deletion} finds the next index where $1$ occurs in a brute-force way.
\end{proof}



\noindent\textbf{Proof of Theorem~\ref{thm:update_sketch_multiple_del}:}
\begin{proof}
Propositions~\ref{prop:prop3333},~\ref{prop:prop4444},~~\ref{prop:prop5555} completes a proof of the theorem. 
\end{proof}

\begin{remark}
Since $\pi$ is minwise-independent, the probability that \(h_{old}\) is in \(\pi(M)\) is
\(O(|M|/d).\)
Thus the expensive part of Algorithm~\ref{alg:update_sketch_multiple_deletion} 
 (line \ref{alg:line11multipleDropHash}) happens rarely. 
\end{remark}

\section{Experiments}
\label{sec:experiments}
\noindent\textbf{Hardware description:}  CPU model name: Intel(R) Xeon(R) CPU @ 2.20GHz; RAM:12.72GB; Model name: Google Colab.
 
\noindent\textbf{Datasets and baselines:} We perform our experiments  on \textit{``Bag-of-Words"} representations of text documents~\citep{UCI}. We use  the following datasets: NYTimes news articles (number of points = $300000$, dimension = $102660$),
Enron emails (number of points = $39861$, dimension= $28102$), and
KOS blog entries (number of points = $3430$, dimension = $6960$).

We consider the binary version of the data, where we focus on the presence/absence of a word in the document. For our experiments, we considered a random sample of $500$ points from the NYTimes and $2000$ points for Enron and  KOS.
 We compare the performance of our algorithms  $\mathrm{multipleLiftHash}$ and $\mathrm{multipleDropHash}$ with respect to running $\minhash$ from scratch on the updated dimension, and we refer to it as vanilla $\minhash$. We also note the performance of sequential versions of single feature insertion/deletion algorithms --  $\mathrm{liftHash}$ and $\mathrm{dropHash}$, respectively. {We  give implementation details of the baseline algorithms as the following link \url{https://tinyurl.com/y98yh6k3}.}
 \begin{table}[H]
\caption{Speedup of 
our algorithms \textit{w.r.t} their vanilla $\minhash$ version.}\label{tab:speedup}
\centering
 \resizebox{\textwidth}{!}{
\begin{tabular}{|c|c|cc|cc|cc|}
\toprule
     \multirow{2}{*}{Experiment}&\multirow{2}{*}{Method}&
      \multicolumn{2}{|c|}{NYTimes} &
       \multicolumn{2}{|c|}{Enron}&
       \multicolumn{2}{|c|}{KOS}\\
      &{}& {Max.} & {Avg.}  & {Max.} & {Avg.}    & {Max.} & {Avg.}\\
      \midrule
  \text{Feature}&$\mathrm{multipleLiftHash}$  &$54.91\times$     &$51.96\times$   &$9.61\times$     &$9.17\times$           &$24.4\times$      & $23.11\times$            \\
     \text{Insertions}&$\mathrm{liftHash}$   &$91.23\times$ & $87.38\times$ & $13.96\times$ & $12.66\times$ &   $35.00\times$ & $35.50\times$  \\
      \midrule
 \text{Feature}&  $\mathrm{multipleDropHash}$ & $109.5\times$& $105.31\times$     &$18.6\times$       & $17.01\times$      &  $46.02\times$ &$43.94\times$ \\  
\text{Deletions}&$\mathrm{dropHash}$   & $78.34\times$ & $72.79\times$ & $15.95\times$ & $14.89\times$ &$38.24\times$      &$35.71\times$   \\
\bottomrule
  \end{tabular}
}
\end{table}

\begin{figure*}[h!]
 \begin{center}
\includegraphics[height=3cm,width=\textwidth]{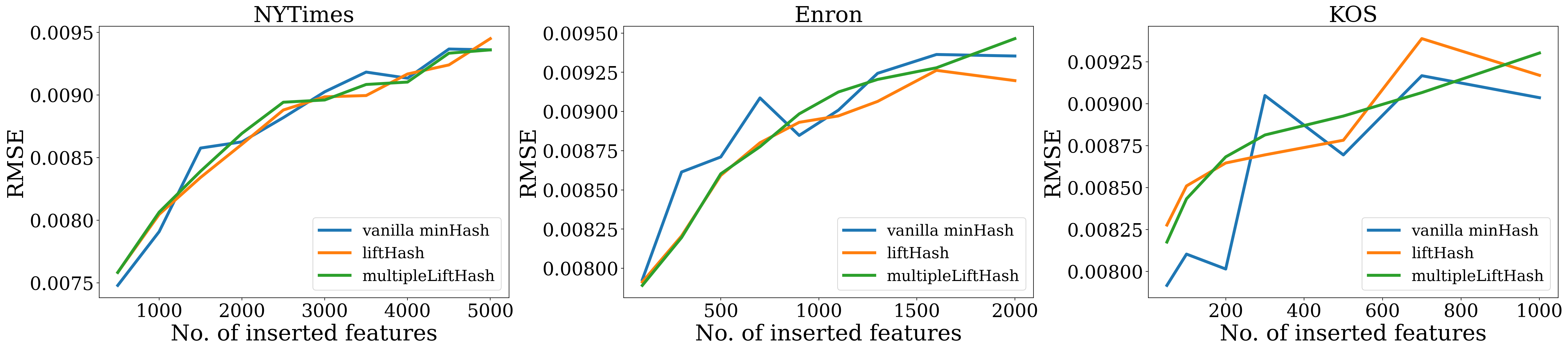}
\includegraphics[height=3cm,width=\textwidth]{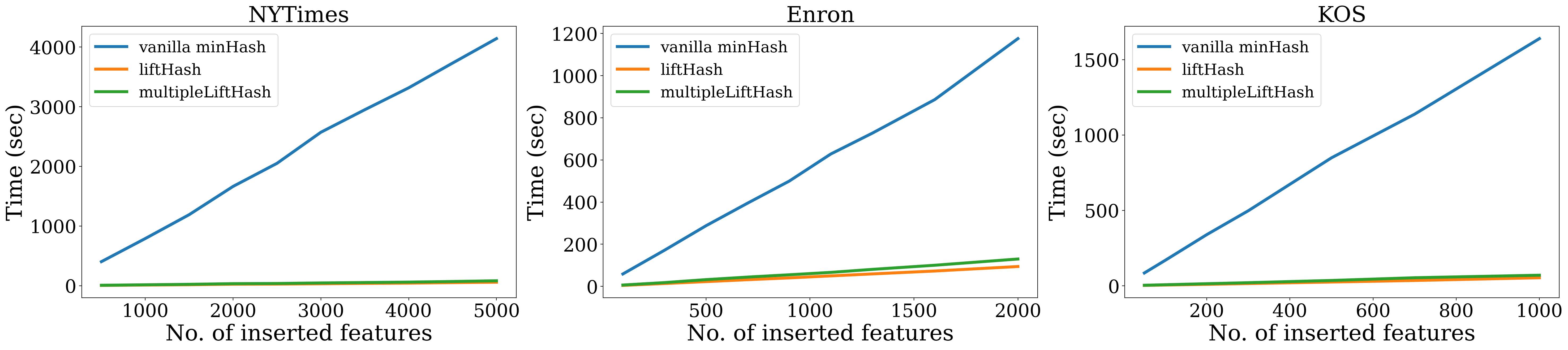}
\caption{{Comparison among $\lifthash$, $\mlifthash$, and vanilla $\minhash$ on the task of feature insertions. Vanilla $\minhash$ corresponds to computing $\minhash$ on the updated dimension.  We iteratively run $\mathrm{liftHash}$ $n$ times, where $n$ is the number of inserted features.  
}  
} \label{fig:feature_insertion_exp_plot}
\end{center}
\end{figure*}

 \begin{figure*}[h!]
\begin{center}
\includegraphics[height=3cm,width=\textwidth]{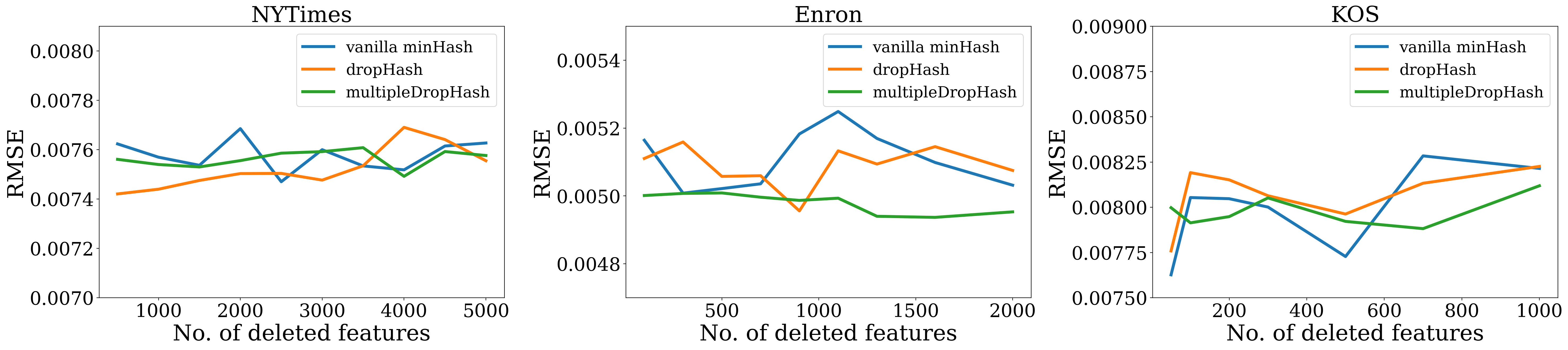}
\includegraphics[height=3cm,width=\textwidth]{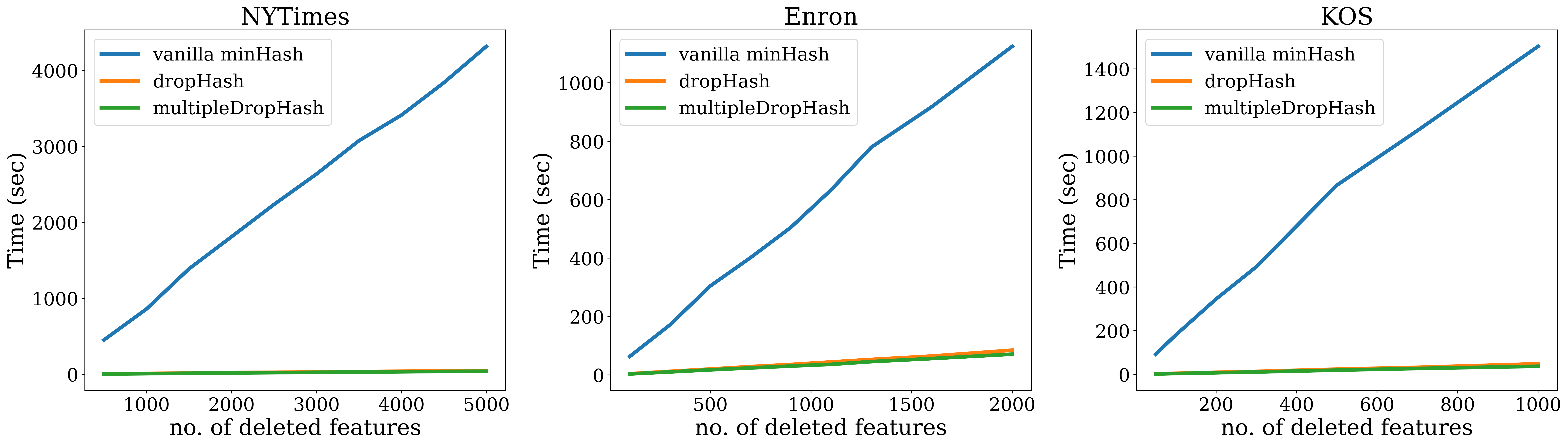}
\caption{{Comparison among $\drophash$, $\mdrophash$, and vanilla $\minhash$ on  feature deletions.  We iteratively run $\mathrm{dropHash}$ $n$ times, where $n$ is the number of deleted features.
}}  \label{fig:feature_deletion_exp_plot}
\end{center}
\end{figure*}

\subsection{Experiments for feature insertions:}\label{subsec:feature_insertion_exp}
We use two metrics for evaluation: a) $\RMSE$: to examine the quality of the sketch, and b) running time: to measure the efficiency.  For each dataset, we first create a $500$ dimensional $\minhash$ sketch using 500 independently generated permutations. 
Consider that we have a set of $n$ random indices representing the locations where features need to be inserted.For each position, we insert the bit $1$ with probability $0.1$ and $0$ with probability $0.9.$ We then run the $\lifthash$ algorithm (Algorithm~\ref{alg:update_sketch}) after each feature insertion, we repeat this step until $n$ feature insertions are done. This gives a $\minhash$ sketch corresponding to the $\lifthash$ algorithm. We again run our $\mathrm{multipleLiftHash}$ algorithm (Algorithm~\ref{alg:update_sketch_multiple_insertion}) on the initial $500$ dimensional sketch with the parameter $n$. We compare our methods with vanilla $\minhash$ by generating a $500$ dimensional sketch corresponding to the updated datasets after feature insertions.  

For computing the $\RMSE$, our ground truth is the pairwise Jaccard similarity on the original full-dimensional data. We measure it by computing the square root of the mean (over all pairs of sketches) of the square of the difference between the pairwise ground truth similarity and the corresponding similarity estimated from the sketch. A  lower RMSE is an indication of better performance. We compare the $\RMSE$ of our methods with that of vanilla $\minhash$ by generating a fresh $500$ dimensional sketch.  We summarise our results in Figure~\ref{fig:feature_insertion_exp_plot}.

\noindent\textbf{Insights:} Both of our algorithms offer comparable performance (under $\RMSE$) with respect to running $\minhash$ from scratch on the updated dimension. That is, our estimate of the Jaccard similarity is as accurate as the one obtained by computing $\minhash$  from scratch on the updated dimension.  Simultaneously, we obtain significant speedups in running time compared to running $\minhash$  from scratch. In particular, the speedup for $\mathrm{multipleLiftHash}$ is noteworthy (Table~\ref{tab:speedup}).

\subsection{Experiments for feature deletion:}\label{subsec:feature_deletion_exp}
We use the same metric as feature insertion experiments -- $\RMSE$ and
running time. For each dataset, we first create a $500$ dimensional $\minhash$ sketch using $\minhash$. Suppose we have a list of $n$ indices that denote the position where features need to be deleted. We then run $\drophash$ algorithm (Algorithm~\ref{alg:update_sketch_del}) after  each feature deletion.  We repeat this step $n$ times. This gives a $\minhash$ sketch corresponding to the $\drophash$ algorithm. We again run our $\mathrm{multipleDropHash}$ algorithm (Algorithm~\ref{alg:update_sketch_multiple_deletion}) on the initial $500$ dimensional sketch with the parameter $n$.  We compare our results with vanilla $\minhash$ by generating a fresh $500$ dimensional sketch on the updated dataset. We note the $\RMSE$ and running time as above. We summarise our results in Figure~\ref{fig:feature_deletion_exp_plot}.

\noindent\textbf{Insights:} Again, both our algorithms offer comparable performance (under $\RMSE$) with respect to running $\minhash$ from scratch. Similar to the previous case, we obtained a significant speedup in running time \textit{w.r.t.} computing $\minhash$  from scratch. In particular, the speedup obtained in $\mathrm{multipleDropHash}$ is quite prominent. We summarise a numerical speedup in Table~\ref{tab:speedup}.

\begin{remark}
Our current implementation of $\mathrm{multipleLiftHash}$ makes multiple passes over indices to be inserted, whereas $\mathrm{multipleDropHash}$ makes only one pass over the deleted indices. This is reflected in higher speedup values for $\mathrm{multipleDropHash}$ in Table \ref{tab:speedup}. We believe an optimized implementation for $\mathrm{multipleLiftHash}$ would further improve the speedup.

\end{remark}

\section{Conclusion and open questions}\label{sec:conclusion}
 

{\color{black} We present algorithms that make $\minhash$  adaptable to dynamic feature insertions and deletions of features. Our proposals' advantage is that they do not require generating fresh permutations to compute the updated sketch. Our algorithms take the current permutation (or its representation using universal hash function~\citep{DBLP:books/daglib/0023376}), $\minhash$ sketch, position, and the corresponding values of inserted/deleted features and output updated sketch. The running time of our algorithms remains linear in the number of inserted/deleted features. We comprehensively analyse our proposals and complement them with supporting experiments on several real-world datasets. Our algorithms are simple, efficient,  and accurately estimate the underlying pairwise Jaccard similarity.} Our work leaves the possibility of  several interesting open questions:
\begin{itemize}
\item extending our results for dense datasets in the case of feature insertions;
\item extending our algorithms for the case when features are inserted/deleted adversely;
\item improving our algorithms when we have prior information about the distribution of features; for example features distribution follows \textit{Zipf’s law} etc;
\item improving theoretical guarantees and obtaining further speedups by optimizing our algorithms.
\end{itemize}

\section*{Acknowledgement:}
We sincerely thank Biswadeep Sen for providing their valuable input on the initial draft of the paper. 

\bibliography{reference}
\pagebreak
\appendix
\end{document}